\documentclass{article}

\PassOptionsToPackage{numbers, compress}{natbib}
\usepackage[preprint]{nips_2018}

\usepackage[utf8]{inputenc} % allow utf-8 input
\usepackage[T1]{fontenc}    % use 8-bit T1 fonts
\usepackage{hyperref}       % hyperlinks
\usepackage{url}            % simple URL typesetting
\usepackage{booktabs}       % professional-quality tables
\usepackage{amsfonts}       % blackboard math symbols
\usepackage{nicefrac}       % compact symbols for 1/2, etc.
\usepackage{microtype}      % microtypography

\pdfoutput=1
\usepackage{graphicx}
\usepackage{subfig}
\usepackage{float}
\usepackage{booktabs} % for professional tables
\usepackage{xspace, amsmath, amssymb, amsthm, bbm, bm}
\usepackage{paralist}
\usepackage{todonotes}
\usepackage{algorithm}
\usepackage{algorithmic}
\usepackage{dsfont}
\allowdisplaybreaks % This allows breaking multi-line equations

\usepackage{etoolbox}
\newcommand{\zerodisplayskips}{%
    \setlength{\abovedisplayskip}{4pt}%
    \setlength{\belowdisplayskip}{4pt}%
    \setlength{\abovedisplayshortskip}{4pt}%
    \setlength{\belowdisplayshortskip}{4pt}}
\appto{\normalsize}{\zerodisplayskips}
\appto{\small}{\zerodisplayskips}
\appto{\footnotesize}{\zerodisplayskips}

\newtheorem{lem}{Lemma}
\newtheorem{thm}{Theorem}
\newtheorem{prop}{Proposition}
\newtheorem{cor}{Corollary}

\newcommand{\eqnref}[1]{Eqn.~\ref{#1}}
\newcommand{\lemref}[1]{Lemma~\ref{#1}}
\newcommand{\corref}[1]{Corollary~\ref{#1}}
\newcommand{\algref}[1]{Algorithm~\ref{#1}}
\newcommand{\thmref}[1]{Theorem~\ref{#1}}
\newcommand{\secref}[1]{Section~\ref{#1}}

\newcommand{\defeq}{:=}
\def\hmu{{\ensuremath{\hat{\mu}}\xspace} }
\newcommand{\EE}{\mathbb{E}}
\newcommand{\VV}{\mathbb{V}}
\newcommand{\PP}{\mathbb{P}}

\def\lt{\left}
\def\rt{\right}
\def\tilN{\ensuremath{\widetilde N}}
\def\eps{\ensuremath{\epsilon}}
\newcommand{\fr}[2]{ { \frac{#1}{#2} }}
\newcommand{\dt}{\delta}
\def\one{\ensuremath{\mathds{1}}} 
\newcommand{\kjunbox}[1]{\vspace{-1ex}{\begin{center}\fbox{%
    \begin{minipage}{.9\textwidth} % line text column
      \centering #1
    \end{minipage}%
  }\end{center}}\vspace{-1ex}}
\makeatletter
\renewcommand{\paragraph}{%
  \@startsection{paragraph}{4}%
  {\z@}{0.20ex \@plus 1ex \@minus .2ex}{-1em}%
  {\normalfont\normalsize\bfseries}%
}
\makeatletter

\title{Adversarial Attacks on Stochastic Bandits}

\author{
  Kwang-Sung Jun\\
  UW-Madison\\
  \texttt{kjun@discovery.wisc.edu}\\
  \And
  Lihong Li\\
  Google Brain \\
  \texttt{lihong@google.com} \\
  \AND
  \hspace{2em}Yuzhe Ma \\
  \hspace{2em}UW-Madison \\
  \hspace{2em}\texttt{ma234@wisc.edu} \\
  \And
  \hspace{2.3em}Xiaojin Zhu\\
  \hspace{2.3em}UW-Madison \\
  \hspace{2.3em}\texttt{jerryzhu@cs.wisc.edu} \\
}

\begin{document}

\maketitle

\begin{abstract}
We study adversarial attacks that manipulate the reward signals to control the actions chosen by a stochastic multi-armed bandit algorithm.
We propose the first attack against two popular bandit algorithms: $\epsilon$-greedy and 
UCB, \emph{without} knowledge of the mean rewards.
The attacker is able to spend only logarithmic effort, multiplied by a problem-specific parameter that becomes smaller as the bandit problem gets easier to attack. 
The result means the attacker can easily hijack the behavior of the bandit algorithm to promote or obstruct certain actions, say, a particular medical treatment.
As bandits are seeing increasingly wide use in practice, our study exposes a significant security threat.
\end{abstract}

%\vspace{-3ex}
\section{Introduction}
\label{sec:intro}
%\vspace{-1ex}
Designing trustworthy machine learning systems requires understanding how they may be attacked.
There has been a surge of interest on adversarial attacks against supervised learning~\cite{goodfellow2014explaining,joseph_nelson_rubinstein_tygar_2018}. 
In contrast, little is known on adversarial attacks against stochastic multi-armed bandits (MABs), a form of online learning with limited feedback.
This is potentially hazardous since stochastic MABs are widely used in the industry to recommend news articles~\cite{li10contextual}, display advertisements~\cite{chapelle14simple}, improve search results~\cite{kveton15cascading}, allocate medical treatment~\cite{kuleshov14algorithms}, and promote users' well-being~\cite{greenewald17action}, among many others.
Indeed, as we show, an adversarial attacker can modify the reward signal to manipulate the MAB for nefarious goals.

Our main contribution is an analysis on reward-manipulation attacks.  
We distinguish three agents in this setting:
``the world,'' 
``Bob'' the bandit algorithm,
and ``Alice'' the attacker. 
As in standard stochastic bandit problems, the world consists of $K$ arms with sub-Gaussian rewards centered at $\mu_1, \ldots, \mu_K$.
Note that we do \emph{not} assume $\{\mu_i\}$ are sorted.
Neither Bob nor Alice knows $\{\mu_i\}$.
Bob pulls selected arms in rounds and attempts to minimize his regret.
When Bob pulls arm $I_t \in [K]$ in round $t$, the world generates a random reward $r^0_t$ drawn from a sub-Gaussian distribution with expectation $\mu_{I_t}$.
However, Alice sits in-between the world and Bob and manipulates the reward into $r_t = r^0_t - \alpha_t$.
We call $\alpha_t \in \mathbb R$ the attack.
If Alice decides not to attack in this round, she simply lets $\alpha_t=0$.
Bob then receives $r_t$, without knowing the presence of Alice.
Without loss of generality, assume arm $K$ is a suboptimal ``attack target'' arm: $\mu_K < \max_{i=1...K} \mu_i$. 
Alice's goal is to manipulate Bob into pulling arm $K$ very often while making small attacks.
Specifically, we show Alice can force Bob to pull the target arm $T-o(T)$ number of times with a cumulative attack cost of $\sum_{t=1}^T |\alpha_t| = O(\log(T))$.

The assumption that Alice does not know $\{\mu_i\}$ is significant because otherwise Alice can perform the attack trivially.
To see this, with the knowledge of $\{\mu_i\}$ Alice would be able to compute the truncated reward gap $\Delta^\eps_i = \max\{ \mu_i-\mu_K + \eps, 0 \} \ge 0$ for all non-target arms $i \ne K$ for some small parameter $\eps>0$.
Alice can perform the following \emph{oracle attack}: in any round where a non-target arm $I_t \neq K$ is pulled, attack with $\alpha_t = \Delta^\eps_{I_t}$.
This oracle attack transforms the original bandit problem into one where all non-target arms have expected reward less than $\mu_K$.
It is well-known that if Bob runs a sublinear-regret algorithm (e.g., UCB~\cite{auer02finite,bubeck12regret}), almost all arm pulls will concentrate on the now-best target arm $K$ in the transformed bandit problem.
Furthermore, Alice's cumulative attack cost will be sublinear in time, because the total number of non-target arm pulls is sublinear in the transformed problem.
In practice, however, it is almost never the case that Alice knows $\mu_1, \ldots, \mu_K$ and hence the $\Delta^\eps_i$'s.
Thus the oracle attack is impractical.
Our focus in this paper is to design an attack that nearly matches the oracle attack, but for Alice who does not know $\{\mu_i\}$.
We do so for two popular bandit algorithms, $\eps$-greedy~\cite{auer02nonstochastic} and UCB~\cite{bubeck12regret}.

What damage can Alice do in practice?
She can largely control the arms pulled by Bob. 
She can also control which arm appears to Bob as the best arm at the end.
As an example, consider the news-delivering contextual bandit problem~\cite{li10contextual}.
The arms are available news articles, and Bob selects which arm to pull (i.e., which article to show to a user at the news site). 
In normal operation, Bob shows news articles to users to maximize the click-through rate.
However, Alice can attack Bob to change his behavior.
For instance, Alice can manipulate the rewards so that users from a particular political base are always shown particular news articles that can reinforce or convert their opinion.
Conversely, Alice can coerce the bandit to not show an important article to certain users.
As another example, Alice may interfere with clinical trials~\cite{kuleshov14algorithms} to funnel most patients toward certain treatment, 
or make researchers draw wrong conclusions on whether treatment is better than control.  
Therefore, adversarial attacks on MAB deserve our attention.
Insights gained from our study can be used to build defense in the future.

Finally, we note that our setting is motivated by modern industry-scale applications of contextual bandits, where arm selection, reward signal collection, and policy updates are done in a distributed way~\cite{agarwal16multiworld,li10contextual}.  Attacks can happen when the reward signal is joined with the selected arm, or when the arm-reward data is sent to another module for Bob to update his policy.  In either case, Alice has access to both $I_t$ and $r_t^0$ for the present and previous rounds.

The rest of the paper is organized as follows.
In \secref{sec:prelim}, we introduce notations and straightforward attack algorithms that serve as baseline.
We then propose our two attack algorithms for $\eps$-greedy and UCB in Section~\ref{sec:egreedy} and~\ref{sec:ucb} respectively, along with their theoretical attack guarantees.
In \secref{sec:sim}, we empirically confirm our findings with toy experiments.
Finally, we conclude our paper with related work (Section~\ref{sec:related}) and a discussion of future work (Section~\ref{sec:con}) that will enrich our understanding of security vulnerability and defense mechanisms for secure MAB deployment.

\section{Preliminaries}
\label{sec:prelim}

Before presenting our main attack algorithms,
in this section we first discuss a simple heuristic attack algorithm which serves to illustrate the intrinsic difficulty of attacks.
Throughout, we assume Bob runs a bandit algorithm with sublinear pseudo-regret 
$\EE \sum_{t=1}^{T} (\max_{j=1}^K \mu_j - \mu_{I_t}) $.
As Alice does not know $\{\mu_i\}$ she must rely on the empirical rewards up to round $t-1$ to decide the appropriate attack $\alpha_t$.
The attack is online since $\alpha_t$ is computed on-the-fly as $I_t$ and $r^0_t$ are revealed. 
The attacking protocol is summarized in \algref{alg:attackeps}.
\begin{algorithm}[H]%[H]
  \begin{algorithmic}[1]
    \caption{Alice's attack against a bandit algorithm}
    \label{alg:attackeps}
    \STATE \textbf{Input}: Bob's bandit algorithm, target arm $K$
    \FOR{$t=1,2,\ldots$}
    \STATE Bob chooses arm $I_t$ to pull.
    \STATE World generates pre-attack reward $r_t^0$.
    \STATE Alice observes $I_t$ and $r_t^0$, and then decides the attack $\alpha_t$. 
    \STATE Alice gives $r_t=r^0_t-\alpha_t$ to Bob.
    \ENDFOR
  \end{algorithmic}
\end{algorithm}

We assume all arm rewards are $\sigma^2$-sub-Gaussian where $\sigma^2$ is known to both Alice and Bob.
Let $N_i(t)$ be the number of pulls of arm $i$ up to round $t$. 
We say the attack is \emph{successful} after $T$ rounds if the number of target-arm pulls is $N_K(T) = T - o(T)$ in expectation or  with high probability,
while minimizing the cumulative attack cost $\sum_{t=1}^T |\alpha_t|$.  
Other attack settings are discussed in~\secref{sec:con}.

For convenience we define the following quantities:
\begin{compactitem}
  \item{$\tau_i(t) \defeq \{s: s \le t, I_s=i\}$, the set of rounds up to $t$ where arm $i$ is chosen,}
  \item{$\hmu^0_i(t) \defeq N_i(t)^{-1} \sum_{s\in \tau_i(t)} r_s^0$, the pre-attack average reward of arm $i$ up to round $t$, and}
  \item{$\hmu_i(t) \defeq N_i(t)^{-1} \sum_{s\in \tau_i(t)} r_s$, the corresponding post-attack average reward.}
\end{compactitem}

%\interfootnotelinepenalty=10000
\textbf{The oracle attack, revisited}\quad
While the oracle attack was impractical, it gives us a baseline for comparison.
The oracle attack drags down the reward of all non-target arms,%
\footnote{The opposite strategy is to push up the target arm: $\alpha_t = \one\{I_t = K\}\cdot(\mu_K -\max_j \mu_j - \eps)$ to make arm $K$ the best arm in post-attack rewards.  However, a successful attack means that Alice pulls the target arm $T - o(T)$ times; the attack cost is necessarily linear in $T$, which is inefficient.  Simulations that support ``drag down'' instead of ``push up'' are presented in Appendix~\ref{sup:constant}.
}
and can be written as
$$\alpha_t = \one\{ I_t \neq K \} \cdot \Delta^\eps_{I_t} \;.$$
%\lihong{I like this more transparent notation.}
Proposition~\ref{thm:oracle} shows that the oracle attack succeeds and requires only a logarithmic attack cost.
While more general statements for sublinear-regret algorithms can be made, we focus on logarithmic-regret bandit algorithms for simplicity.
Throughout, omitted proofs can be found in our supplementary material.
\begin{prop}\label{thm:oracle}
  Assume that Bob's bandit algorithm achieves an $O(\log T)$ regret bound.
  Then, Alice's oracle attack with $\eps > 0$ succeeds; i.e., $\EE N_K(T) = T - o(T)$.
  Furthermore, the expected total attack cost is $O\left(\sum_{i=1}^{K-1} \Delta^\eps_i \log T\right)$.\footnote{For near-optimal algorithms like UCB~\cite{auer02finite}, one can find the optimal choice of $\eps$. See our supplementary material for detail.}
%  \kwang{Can we make any lowerbound argument on this?}
\end{prop}

\textbf{The heuristic constant attack}\quad
A slight variant of the oracle attack is to attack all the non-target arms with a single constant amount $A > 0$, regardless of the actual $\mu_i$'s:
$$\alpha_t = \one\{ I_t \neq K \} \cdot A.$$
Let $\Delta_i := \Delta^0_i$. 
Unfortunately, this heuristic constant attack depends critically on the value of $A$ compared to the unknown maximum gap $\max_i \Delta_i$.
Proposition~\ref{thm:constant} states the condition under which the attack succeeds: 
\begin{prop}\label{thm:constant}
  Assume that Bob's bandit algorithm achieves an $O(\log T)$ regret bound.
  Then, Alice's heuristic constant attack with $A$ succeeds if and only if $A > \max_i \Delta_i$.
  If the attack succeeds, then the expected attack cost is $O(A \log T)$.
\end{prop}

Conversely, if $A < \max_i \Delta_i$ the attack fails.  This is because in the transformed bandit problem, there exists an arm that has a higher expected reward than arm $K$, and Bob will mostly pull that arm.
Therefore, the heuristic constant attack has to know an unknown quantity to guarantee a successful attack.
Moreover, the attack is non-adaptive to the problem difficulty since some $\Delta_i$'s can be much smaller than $A$, in which case Alice pays an unnecessarily large attack cost. 

We therefore ask the following question:
\kjunbox{%
  Does there exist an attacker Alice that guarantees a successful attack with cost adaptive to the problem difficulty?
}

%  Ideally, one would like to construct an attacker that can be successful and adaptive for any logarithmic-regret learner.
%  Unfortunately, we show in Theorem~\ref{thm:lowerbound} that this is not possible unless further assumptions are made on the learner.
%  \begin{thm}\label{thm:lowerbound}
%    Suppose the attacker's reward manipulation results in i.i.d. $\sigma^2$-sub-Gaussian post-attack rewards for each arm.
%    Assume that the learner enjoys an $O(\log T)$ regret bound.
%    Then, the attack cannot be successful for every bandit problem instance.
%    Furthermore, even if the attacker is aware of an upper bound $U$ on $\max_i \Delta_i$ and the attack is successful for any bandit instance, the attack cost cannot be adaptive.
%  \end{thm}
%  %\begin{proof}
%  %  See Section~\ref{sec:lowerbound} of our supplementary material.
%  %\end{proof}
%  Attacks that maintain i.i.d.-ness of the post-attack rewards include the static attack.
%  The consequence is that whenever an attacker tries to achieve adaptivity, she has to break the i.i.d.-ness of the post-attack rewards, which invalidates the regret bound of the stochastic MAB learner.
%  We are then left with no conditions on the learner, and little can be said about the effect of an attack strategy in general.
%  
%  In this paper, we overcome such a difficulty by turning to \emph{learner-specific} attack strategies.
The answer is yes.
We present attack strategies against two popular bandit algorithms of Bob: $\epsilon$-greedy and UCB. We show that Alice can indeed succeed in her attacks and incur cost as small as that of the oracle with an additive term due to the sub-Gaussian noise level $\sigma$.

\section{Alice's Attack on \texorpdfstring{$\epsilon$}{}-Greedy Bob}
\label{sec:egreedy}

The $\eps$-greedy strategy initially pulls each arm once in the first $K$ rounds.
For convenience, we assume that the target arm is pulled first: $I_1=K$.
Our results in this section can be adapted to any order of initialization with more complicated notation.

Bob's $\epsilon$-greedy strategy has the following arm-selection rule for $t>K$~\cite{auer02nonstochastic}:
\[
I_t = \begin{cases}
%K+1-t, & \text{if $t \le K$ (pull $K, \ldots, 1$)} \\
\text{draw } \mathrm{uniform}[K], & \text{w.p. $\epsilon_t$} \text{\hspace{1.18em} (exploration)} \\
\arg\max_i \hmu_i(t-1), & \text{otherwise} \text{ (exploitation)}
\end{cases}\,.
\]
The strategy uses an exploration scheme $\{\epsilon_t\}$ over $t$.
Alice's attack algorithm is not aware of $\{\epsilon_t\}$ though her cumulative attack cost $\sum |\alpha_t|$ will implicitly depend on it.
Later in Corollary~\ref{cor:sublinear} we show that, for the typical decaying scheme $\epsilon_t \propto 1/t$, the cumulative attack cost is mild: $O(\log(t))$.

Alice wants to make Bob \emph{always} pull the target arm during \emph{exploitation} rounds. 
Since Alice has no influence on which arm is pulled during exploration, this attack goal is the strongest she can achieve.
Here, \algref{alg:attackeps} is specialized to ensure the following condition:
\begin{equation}
\hat\mu_{I_t}(t) \le \hat\mu_K(t) - 2 \beta(N_K(t)),
\end{equation}
where we define $\beta(N)$ as
\begin{equation}
\beta(N) \defeq \sqrt{\frac{2\sigma^2}{N}\log\frac{\pi^2 K N^2}{3 \delta}}.
\label{eq:beta}
\end{equation}
From this condition, we derive the actual attack $\alpha_t$.  Since 
\begin{equation}
\hat\mu_{I_t}(t) = \frac{\hat\mu_{I_t}(t-1)N_{I_t}(t-1) + r_t^0-\alpha_t}{N_{I_t}(t)},
\end{equation}
we set the attack in \algref{alg:attackeps} as
\begin{equation}  
\alpha_t = \left[ \hat\mu_{I_t}(t-1)N_{I_t}(t-1) + r_t^0 - \left(\hat\mu_{K}(t) - 2 \beta(N_{K}(t))\right) N_{I_t}(t) 
\right]_+\,,
\label{eq:alpha}
\end{equation}
where $[z]_+=\max(0, z)$.
Note $\alpha$ is always non-negative, thus the cumulative attack cost can be written without absolute value: $\sum_{t=1}^T \alpha_t$.

With this $\alpha_t$, we claim that (i) Alice forces Bob to pull the target arm in all exploitation rounds as shown in Lemma~\ref{lem:alwaysK}, and (ii) the cumulative attack cost is logarithmic in $t$ for standard $\epsilon$-greedy learner exploration scheme $\epsilon_t=O(1/t)$ as shown in \corref{cor:sublinear}.
Our main result is the following general upper bound on the cumulative attack cost.
\begin{thm} \label{thm:egreedy-main}
  Let $\delta \le 1/2$.
With probability at least $1-2\delta$, for any $T$ satisfying $\sum_{t=1}^T \eps_t \ge \fr{K}{e-2}\log(K/\dt)$,\footnote{
    One can drop this condition by considering slightly larger $\tilN(t)$ and smaller $\tilN_K(t)$.
    However, we keep the condition as it simplifies $\tilN(t)$ and $\tilN_K(t)$.
    We refer to the proof of \lemref{lem:Nitupperbound} for detail.
  } 
Alice forces Bob running $\epsilon$-greedy to choose the target arm in at least $\tilN_K(T)$ rounds, using a cumulative attack cost at most
\begin{align*}
  \sum_{t=1}^T |\alpha_t|
  < \lt(\sum_{i=1}^K \Delta_i \rt) \tilN(T)     + (K-1)\cdot\lt( \tilN(T)\beta(\tilN(T)) + 3\tilN(T)\beta(\tilN_K(T)) \rt)
%     + \sigma\Bigg( \sqrt{2\tilN(t)\log\lt( \frac{\pi^2K(\tilN(t))^2}{3\delta} \rt)} 
% \\&\quad + \tilN(t)\sqrt{\frac{2}{\tilN_K(t)}\log\lt( \frac{\pi^2K(\tilN_K(t))^2}{3\delta} \rt)} \Bigg)
\end{align*}
where
{ %\small
\begin{eqnarray*}
  \tilN(T)
  &=& \left(\fr{\sum_{t=1}^T {\epsilon_t}}{K}\right) 
   + \sqrt{3 \log\left(\frac{K}{\delta}\right) \left(\fr{\sum_{t=1}^T {\epsilon_t}}{K}\right) }\,, \\
   %\,\,\,
\tilN_K(T)
  &=& T - \lt(\sum_{t=1}^T \eps_t \rt) 
    - \sqrt{3 \log\left(\frac{K}{\delta}\right) \left(\sum_{t=1}^T {\epsilon_t}\right)}\,.
  \end{eqnarray*}}
\end{thm}
Before proving the theorem, we first look at its consequence.
If Bob's $\epsilon_t$ decay scheme is $\epsilon_t=\min\{1, cK/t\}$ for some $c>0$ as recommended in~\citet{auer02nonstochastic}, Alice's cumulative attack cost is $O(\sum_{i=1}^K \Delta_i \log T)$ for large enough $T$,
as the following corollary shows:

\begin{cor} \label{cor:sublinear}
Inherit the assumptions in \thmref{thm:egreedy-main}.
Fix $K$ and $\delta$.
If $\epsilon_t=cK/t$ for some constant $c>0$, then
\begin{equation}
\sum_{t=1}^T | \alpha_t | = \widehat O\left( \lt(\sum_{i=1}^K \Delta_i \rt) \log T + \sigma K\sqrt{\log T}\right),
\end{equation}
where $\widehat O$ ignores $\log \log$ factors.
\end{cor}

Note that the two important constants are $\sum_i \Delta_i$ and $\sigma$.
While a large $\sigma$ can increase the cost significantly, the term with $\sum_i \Delta_i$ dominates the cost for large enough $T$.
Specifically, $\sum_i \Delta_i$ is multiplied by $\log T$ that is of higher order than $\sqrt{\log T}$.
We empirically verify the scaling of cost with $T$ in Section~\ref{sec:sim}.

To prove Theorem~\ref{thm:egreedy-main}, we first show that $\beta$ in~\eqref{eq:beta} is a high-probability bound on the pre-attack empirical mean of all arms on all rounds.
Define the event 
\begin{equation}
E \defeq \{ \forall i, \forall t>K: |\hmu^0_i(t) - \mu_i| < \beta(N_i(t))\}.
\label{eq:band}
\end{equation}

\begin{lem}
\label{lem:band}
For $\delta \in (0,1)$, $\PP\left(E \right) > 1-\delta$.
\end{lem}

The following lemma proves the first half of our claim.
\begin{lem}
\label{lem:alwaysK}
For $\delta\le 1/2$ and under event $E$,
attacks~\eqref{eq:alpha} force Bob to always pull the target arm $K$ in exploitation rounds.
\end{lem}

We now show that on average each attack on a non-target arm $i$ is not much bigger than $\Delta_i$.
\begin{lem} \label{lem:alphaupperbound}
For $\dt\le1/2$ and under event $E$, we have for all arm $i<K$ and all $t$ that
\begin{equation*}
\sum_{s\in \tau_i(t)} |\alpha_s| < \left(\Delta_i + \beta(N_i(t)) + 3 \beta(N_{K}(t))\right) N_{i}(t)\,.
\end{equation*}
\end{lem}

Finally, we upper bound the number of non-target arm $i$ pulls $N_{i}(T)$ for $i<K$. 
Recall the arm $i$ pulls are only the result of exploration rounds.
In round $t$ the exploration probability is $\eps_t$; if Bob explores, he chooses an arm uniformly at random.
We also lower bound the target arm pulls $N_K(T)$.
\begin{lem} \label{lem:Nitupperbound}
  Let $\dt<1/2$. 
  Suppose $T$ satisfy $\sum_{t=1}^T \eps_t \ge \fr{K}{e-2}\log(K/\dt)$.
  With probability at least $1-\delta$, for all non-target arms $i<K$,
  \begin{align*}
    N_i(T)
    < \sum_{t=1}^T \frac{\eps_t}{K} + \sqrt{3\sum_{s=1}^T\frac{\eps_t}{K}\log\frac{K}{\delta}}\,.
  \end{align*}
  and for the target arm $K$,
  \begin{align*}
    N_K(T) > T - \sum_{t=1}^T\eps_t - \sqrt{3\sum_{s=1}^T\eps_t\log\frac{K}{\delta}}\,.
  \end{align*}
\end{lem}

We are now ready to prove Theorem~\ref{thm:egreedy-main}.
\begin{proof}
The theorem follows immediately from a union bound over \lemref{lem:alphaupperbound} and \lemref{lem:Nitupperbound} below. 
We add up the attack costs over $K-1$ non-target arms.
Then, we note that $N \beta(N)$ is increasing in $N$ so $N_i(T) \beta(N_i(T)) \le \tilN(T) \beta(\tilN(T))$.
Finally, by \lemref{lem:beta} in our supplementary material $\beta(N)$ is decreasing in $N$,  so $\beta(N_K(T)) \le \beta(\tilN_K(T))$.
\end{proof}

\section{Alice's Attack on UCB Bob}
\label{sec:ucb}
  
Recall that we assume rewards are $\sigma^2$-sub-Gaussian.  
Bob's UCB algorithm in its basic form often assumes rewards are bounded in $[0,1]$; we need to modify the algorithm to handle the more general sub-Gaussian rewards.  By choosing $\alpha=4.5$ and $\psi:\lambda\mapsto\frac{\sigma^2\lambda^2}{2}$ in the $(\alpha,\psi)$-UCB algorithm of \citet[Section~2.2]{bubeck12regret}, we obtain the following arm-selection rule:
\[
I_t = \begin{cases}
t, & \text{if $t \le K$} \\
\arg\max_i \left\{ \hmu_i(t-1) + 3\sigma\sqrt{\frac{\log t}{N_i(t-1)}} \right\}, & \text{otherwise.}
\end{cases}
\]

For the first $K$ rounds where Bob plays each of the $K$ arms once in an arbitrary order, Alice does not attack: $\alpha_t = 0$ for $t \le K$.  After that, attack happens only when $I_t \ne K$.  Specifically, consider any round $t>K$ where Bob pulls arm $i\ne K$.  It follows from the UCB algorithm that
\[
\hmu_i(t-1)+3\sigma\sqrt{\frac{\log t}{N_i(t-1)}} \ge \hmu_K(t-1)+3\sigma\sqrt{\frac{\log t}{N_K(t-1)}} \,.
\]
Alice attacks as follows.  She computes an attack $\alpha_t$ with the smallest absolute value, such that
\[
\hmu_i(t) \le \hmu_K(t-1) - 2 \beta(N_K(t-1)) - \Delta_0\,,
\]
where $\Delta_0 \ge 0$ is a parameter of Alice.
Since the post-attack empirical mean can be computed recursively by the following
\[
\hmu_i(t) = \frac{N_i(t-1)\hmu_i(t-1)+r_t^0 -\alpha_t}{N_i(t-1)+1}\,,
\]
where $r_t^0$ is the pre-attack reward; this enables us to write down in closed form Alice's attack:
\begin{equation}
\alpha_t = \Big[ N_i(t)\hmu^0_i(t) - \sum_{s\in \tau_i(t-1)}\alpha_s - N_i(t)\cdot\lt( \hmu_K(t-1) - 2  \beta(N_K(t-1)) - \Delta_0\rt)
\Big]_+ \,. \label{eqn:ucb-attack}
\end{equation}
%\begin{eqnarray}
%\alpha_t &=& \big[ N_i(t)\hmu^0_i(t) - \sum_{s\in \tau_i(t-1)}\alpha_s - N_i(t)\hmu_K(t-1) \nonumber \\
%&& + 2 N_i(t) \beta(N_K(t-1)) + N_i(t) \Delta_0
%\big]_+ \,. \label{eqn:ucb-attack}
%\end{eqnarray}
For convenience, define $\alpha_t=0$ if $I_t=K$.
We now present the main theorem on Alice's cumulative attack cost against Bob who runs UCB.

\begin{thm} \label{thm:ucb-main}
Suppose $T \ge 2K$ and $\delta \le 1/2$.  Then, with probability at least $1-\delta$, Alice forces Bob to choose the target arm in at least
%{\small
\begin{align*}
T- (K-1) \left(2 + \frac{9\sigma^2}{\Delta_0^2}\log T\right),
\end{align*}
%}
rounds, using a cumulative attack cost at most
{\small
\begin{align*}
 \sum_{t=1}^T \alpha_t \le  \left(2 + \frac{9\sigma^2}{\Delta_0^2}\log T\right) \sum_{i<K} (\Delta_i + \Delta_0) 
+ \sigma(K-1) \sqrt{ 32(2 + \frac{9\sigma^2}{\Delta_0^2}\log T) \log\frac{\pi^2K(2 + \frac{9\sigma^2}{\Delta_0^2}\log T)^2}{3\delta}  }\,.
\end{align*}
}
\end{thm}

While the bounds in the theorem are somewhat complicated, the next corollary is more interpretable and follows from a straightforward calculation. 
%In particular, for a fixed problem, the attack cost is essentially on the order of $\sum_{i<K}\Delta_i \log T $.
Specifically, we have the following by straightforward calculation:
\begin{cor} \label{cor:ucb-main}
Inherit the assumptions in \thmref{thm:ucb-main} and fix $\delta$.  Then, the total number of non-target arm pulls is
\[
O\left(K + \frac{K\sigma^2}{\Delta_0^2}\log T\right)\,,
\]
and the cumulative attack cost is
{\small
\begin{align*}
  \widehat{O}\Bigg(\lt(1 + \frac{\sigma^2}{\Delta_0^2}\log T \rt) \sum_{i<K} (\Delta_i+\Delta_0) + \sigma K\cdot\lt(1 + \fr{\sigma }{\Delta_0}\sqrt{\log T}\rt)\sqrt{\log\lt(1+\fr{K\sigma}{\Delta_0}\rt)}\Bigg)\,,
\end{align*}
} \hspace{-2mm}
where $\widehat{O}$ ignores $\log \log(T)$ factors.
\end{cor}
We observe that a larger $\Delta_0$ decreases non-target arm pulls (i.e. 
a more effective attack).
The effect diminishes when $\Delta_0 > \sigma\sqrt{\log T}$ since $\frac{K\sigma^2}{\Delta_0^2}\log T<K$.
Thus there is no need for Alice to choose a larger $\Delta_0$.
%\kwang{some added text starting here. please take a look.}%
By choosing $\Delta_0=\Theta(\sigma)$, the cost is $\widehat O( \sum_{i<K} \Delta_i \log T +  \sigma K \log T)$.
This is slightly worse than the cost of attacking $\eps$-greedy where $\sigma$ is multiplied by $\sqrt{\log T}$ rather than $\log T$.
However, we find that a stronger attack is possible when the time horizon $T$ is fixed and known to Alice ahead of time (i.e., the fixed budget setting).
One can show that this choice $\Delta_0 = \Theta\left(\sigma\sqrt{\log T}\right)$ minimizes the cumulative attack cost, which is $\widehat O\left( K \sigma \sqrt{\log T}\right)$.
This is a very strong attack since the dominating term w.r.t. $T$ does not depend on $\sum_{i<K} \Delta_i$; in fact the cost associated with $\sum_{i<K} \Delta_i$ does not grow with $T$ at all.
%\lihong{I haven't checked, but did you mean ``does not grow with $T$ at all''?}
This means that under the fixed budget setting algorithm-specific attacks can be better than the oracle attack that is algorithm-independent.
Whether the same is true in the anytime setting (i.e., $T$ is unknown ahead of time) is left as an open problem.

For the proof of Theorem~\ref{thm:ucb-main} we use the following two lemmas.
\begin{lem} \label{lem:ni-upper-bound}
Assume event $E$ holds and $\delta \le 1/2$.  Then, for any $i<K$ and any $t \ge 2K$, we have
\begin{equation}
N_i(t) \le \min \{N_K(t), 2 + \frac{9\sigma^2}{\Delta_0^2}\log t \}\,.  \label{eqn:ni-ub}
\end{equation}
\end{lem}

\begin{lem} \label{lem:total-cost}
Assume event $E$ holds and $\delta \le 1/2$.  
Then, at any round $t \ge 2K$, the cumulative attack cost to any fixed arm $i<K$ can be bounded as:
\[
\sum_{s \in \tau_i(t)} \alpha_s \le N_i(t) \Big(\Delta_i + \Delta_0 + 4 \beta(N_i(t)) \Big) \,.
\]
\end{lem}
\begin{proof}[Proof of \thmref{thm:ucb-main}]
Suppose event $E$ holds.
The bounds are direct consequences of Lemmas~\ref{lem:total-cost} and \ref{lem:ni-upper-bound} below, by summing the corresponding upper bounds over all non-target arms $i$.  
Specifically, the number of target arm pulls is $T-\sum_{i<K} N_i(T)$, and the cumulative attack cost is
$\sum_{t=1}^T \alpha_t = \sum_{i<K}\sum_{t\in \tau_i(T)}\alpha_t$.
Since event $E$ is true with probability at least $1-\delta$ (\lemref{lem:band}), the bounds also hold with probability at least $1-\delta$.
\vspace{-1ex}
\end{proof}

\section{Simulations}
\label{sec:sim}
In this section, we run simulations on attacking $\epsilon$-greedy and UCB algorithms to illustrate our theoretical findings.% and show that our attack algorithms can efficiently boost the number of target arm pulls.

\textbf{Attacking $\epsilon$-greedy}\quad
The bandit has two arms.
The reward distributions of arms 1 and 2 are $\mathcal{N}(\Delta_1,\sigma^2)$ and $\mathcal{N}(0,\sigma^2)$, respectively, with $\Delta_1>0$. Alice's target arm is arm 2. We let $\delta=0.025$. 
%in Theorem~\ref{thm:egreedy-main} be 0.025 so that it achieves 0.95 confidence on the attack cost. 
Bob's exploration probability decays as $\epsilon_t=\frac{1}{t}$. 
We run Alice and Bob for $T=10^5$ rounds; this forms one trial. 
We repeat $1000$ trials. 
%We compute the cumulative attack cost $C_T^i$ up to round $T$. We then compute the average cost as $C_T=\frac{1}{M}\sum_{i=1}^M C_T^i$. 

In Figure~\ref{exp:egreedy}\subref{exp:epsilon-cost-delta}, we fix $\sigma=0.1$ and show Alice's cumulative attack cost $\sum_{s=1}^t |\alpha_s|$ for different $\Delta_1$ values. 
Each curve is the average over 1000 trials.
These curves demonstrate that Alice's attack cost is proportional to $\log t$ as predicted by Corollary~\ref{cor:sublinear}.
As the reward gap $\Delta_1$ becomes larger, more attack is needed to reduce the reward of arm 1, and the slope increases.

%In Figure~\ref{exp:egreedy}\subref{exp:epsilon-cost-sigma}, we fix $\Delta_1=1$ and show $\log(C_t)$ as a function of $\log \log t$.
Furthermore, note that $\sum_{t=1}^T |\alpha_t|= \widehat O\left(\Delta_1 \log T + \sigma \sqrt{\log T} \right)$. Ignoring $\log\log T$ terms, we have $\sum_{t=1}^T |\alpha_t| \le C(\Delta_1 \log T + \sigma \sqrt{\log T})$ for some constant $C>0$ and large enough $T$. Therefore, $\log\left(\sum_{t=1}^T |\alpha_t|\right) \le \max\{\log\log T + \log \Delta_1 , \tfrac{1}{2} \log\log T + \log \sigma \}+\log C$.
We thus expect the log-cost curve as a function of $\log\log T$ to behave like the maximum of two lines, one with slope $1/2$ and the other with slope $1$.
%\lihong{Are we sure about these absolute constants?  Are they lost in the big-O notation already?}
Indeed, we observe such a curve in Figure~\ref{exp:egreedy}\subref{exp:epsilon-cost-sigma} where we fix $\Delta_1=1$ and vary $\sigma$.
All the slopes eventually approach 1, though larger $\sigma$'s take a longer time.
This implies that the effect of $\sigma$ diminishes for large enough $T$, which was predicted by Corollary~\ref{cor:sublinear}.

In Figure~\ref{exp:egreedy}\subref{exp:epsilon-TAP}, we compare the number of target arm (the suboptimal arm 2) pulls with and without attack. 
This experiment is with $\Delta_1=0.1$ and $\sigma=0.1$.
Alice's attack dramatically forces Bob to pull the target arm.
In $10000$ rounds, Bob is forced to pull the target arm $9994$ rounds with the attack,  compared to only $6$ rounds if Alice was not present.

\begin{figure}[t]
\centering
\subfloat[Attack cost $\sum_{s=1}^t |\alpha_s|$ as $\Delta_1$ varies]{\label{exp:epsilon-cost-delta}
\includegraphics[width=0.3\textwidth, height=0.22\textwidth]{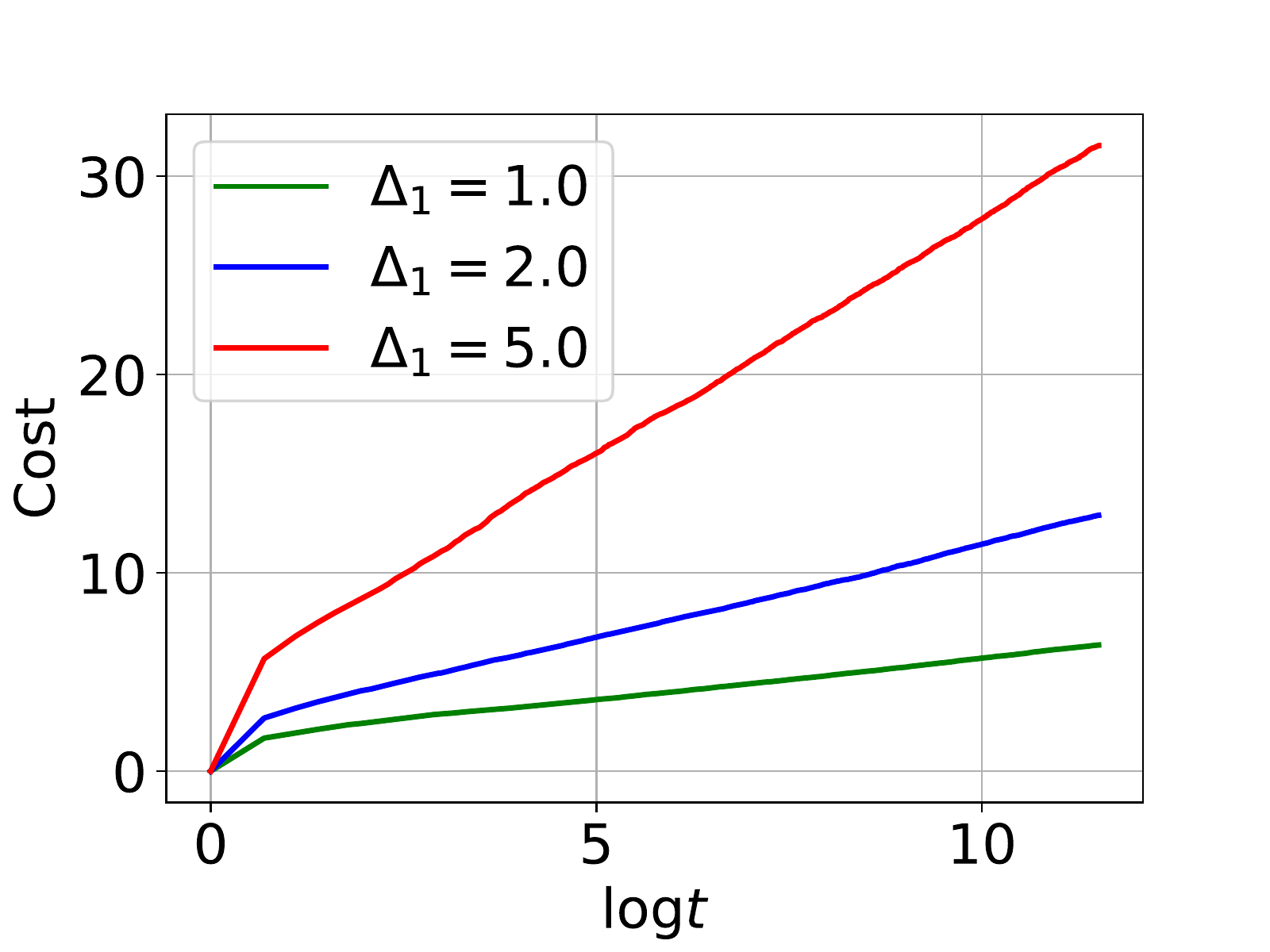}
}
\hspace{1em}
\subfloat[Attack cost as $\sigma$ varies; dotted lines depict slope 1/2 and 1 for comparison.]{\label{exp:epsilon-cost-sigma}
\includegraphics[width=0.3\textwidth, height=0.22\textwidth]{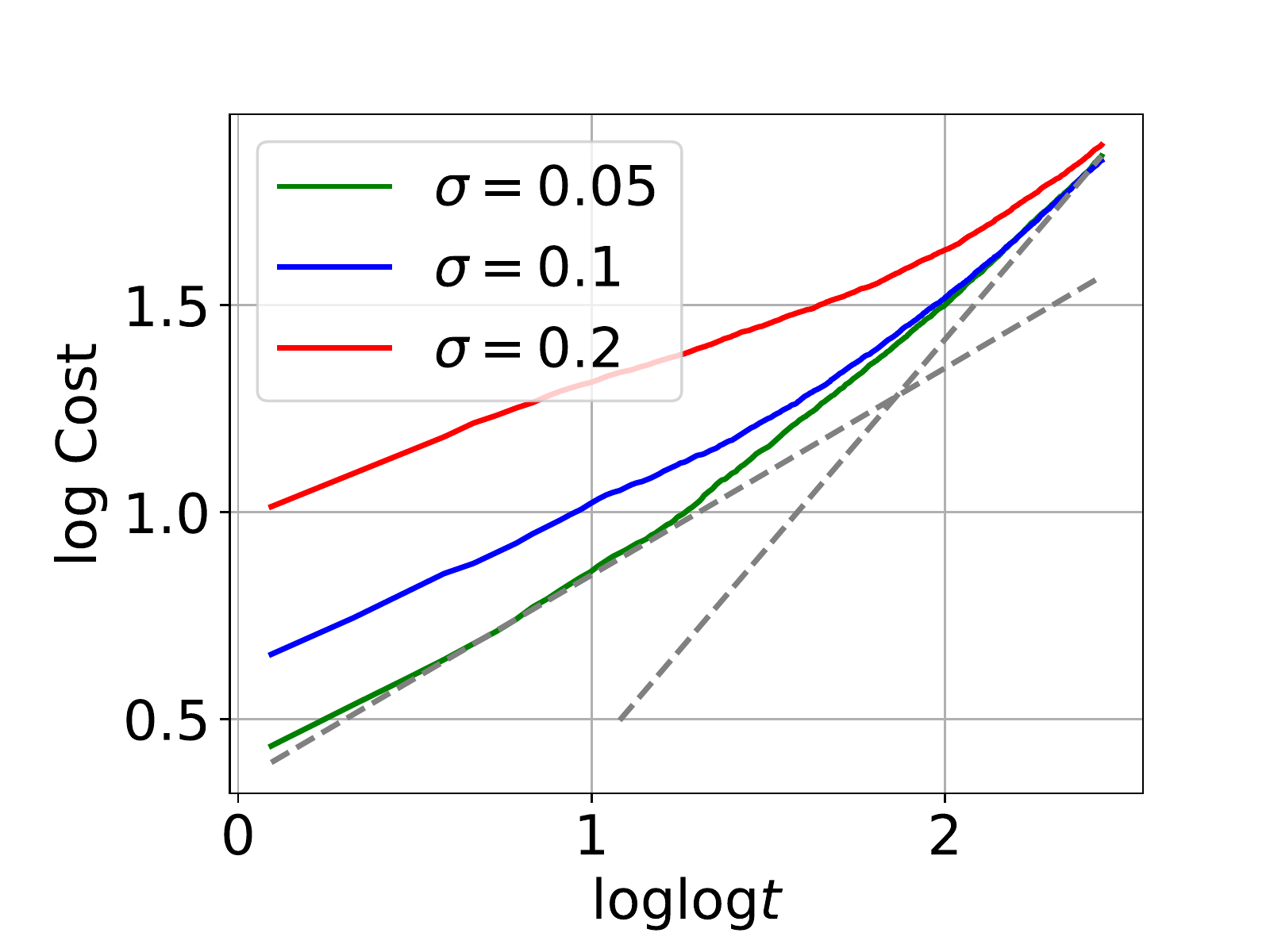}
}
\hspace{1em}
\subfloat[Target arm pulls $N_K(t)$]{\label{exp:epsilon-TAP}
\includegraphics[width=0.3\textwidth, height=0.22\textwidth]{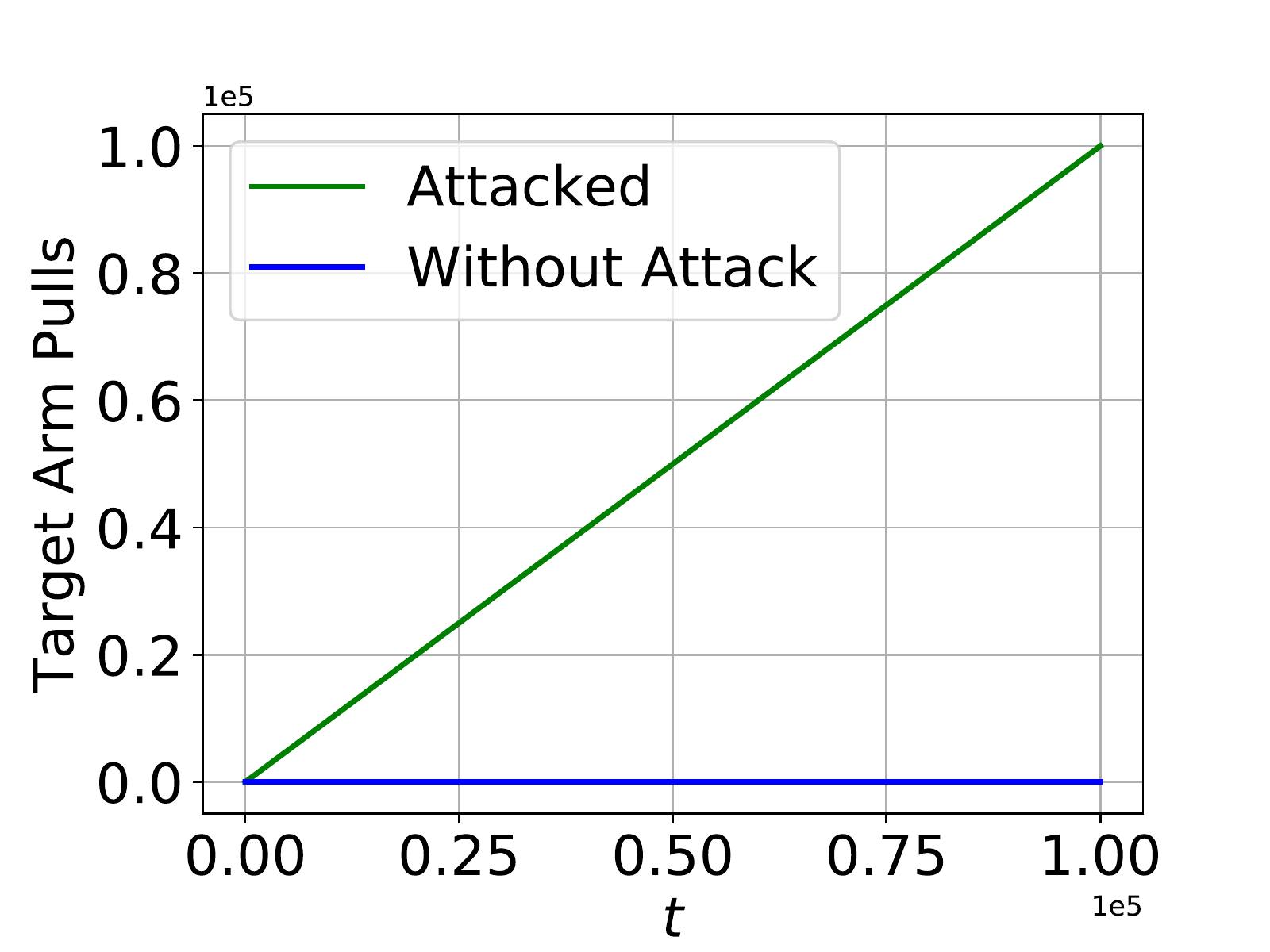}
}
\caption{Attack on $\epsilon$-greedy bandit.}\label{exp:egreedy}
%\vspace{-.5ex}
\end{figure}

\textbf{Attacking UCB}\quad
The bandit has two arms. The reward distributions are the same as the $\epsilon$-greedy experiment. 
We let $\delta=0.05$.
To study how $\sigma$ and $\Delta_0$ affects the cumulative attack cost, we perform two groups of experiments.
In the first group, we fix $\sigma=0.1$ and vary Alice's free parameter $\Delta_0$ while in the second group, we fix $\Delta_0=0.1$ and vary $\sigma$. 
We perform $100$ trials with $T=10^7$ rounds.

Figure~\ref{exp:UCB_cost}\subref{exp:UCB_cost_VD} shows Alice's cumulative attack cost as $\Delta_0$ varies. 
As $\Delta_0$ increases, the cumulative attack cost decreases. 
In Figure~\ref{exp:UCB_cost}\subref{exp:UCB_cost_VS}, we show the cost as $\sigma$ varies. 
Note that for large enough $t$, the cost grows almost linearly with $\log t$, which is implied by Corollary~\ref{cor:ucb-main}. In both figures, there is a large attack near the beginning, after which the cost grows slowly. 
This is because the initial attacks drag down the empirical average of non-target arms by a large amount, such that the target arm appears to have the best UCB for many subsequent rounds.  
Figure~\ref{exp:UCB_cost}\subref{exp:UCB_TAP} again shows that Alice's attack forces Bob to pull the target arm: with attack Bob is forced to pull the target arm $10^7-2$ times, compared to only $156$ times without attack.

\begin{figure*}
\centering
\subfloat[Attack cost $\sum_{s=1}^t \alpha_s$ as $\Delta_0$ varies]{
\includegraphics[width=0.3\textwidth, height=0.22\textwidth]{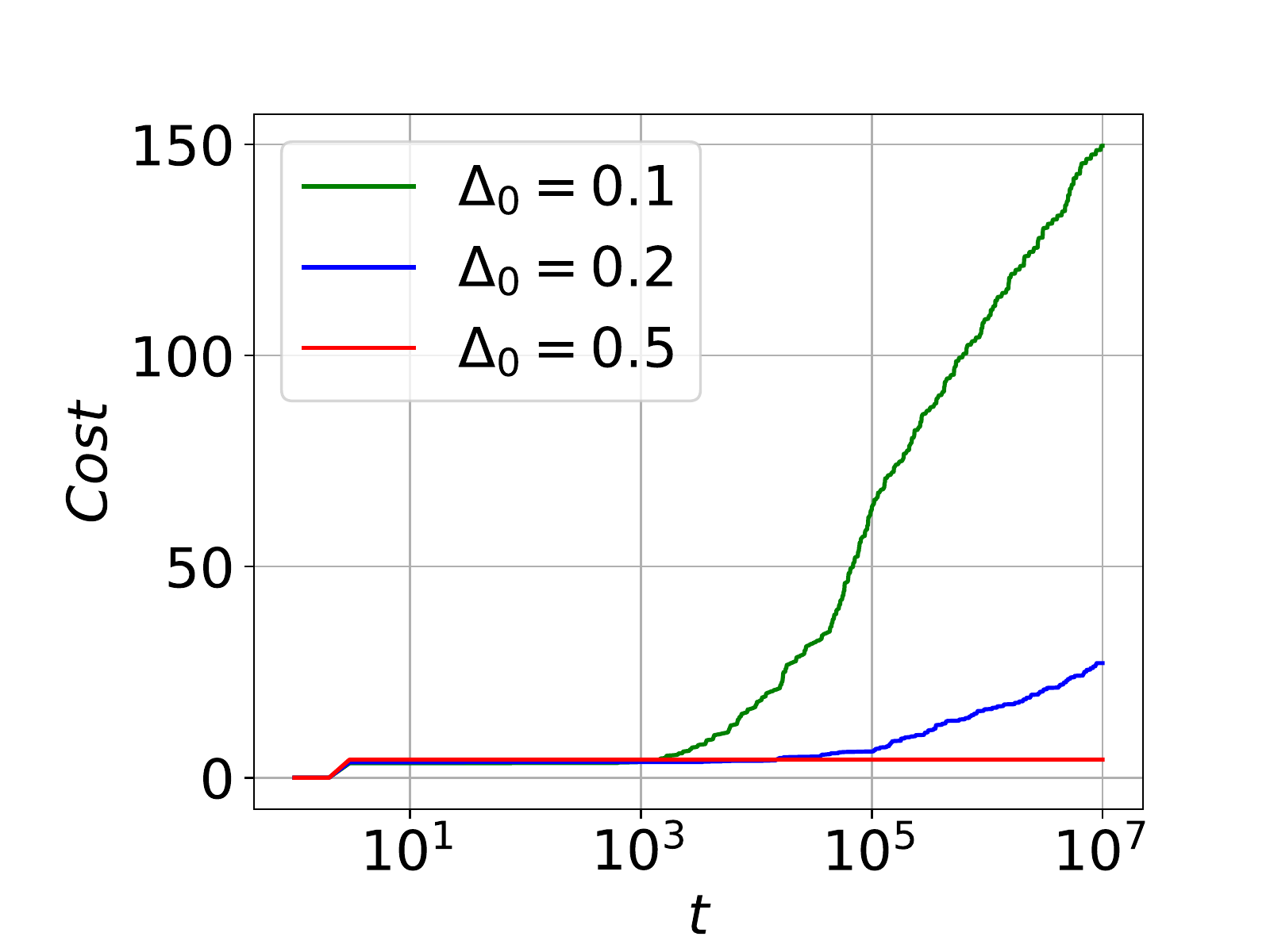}\label{exp:UCB_cost_VD}
}
\hspace{1em}
\subfloat[Attack cost as $\sigma$ varies]
{
\includegraphics[width=0.3\textwidth, height=0.22\textwidth]{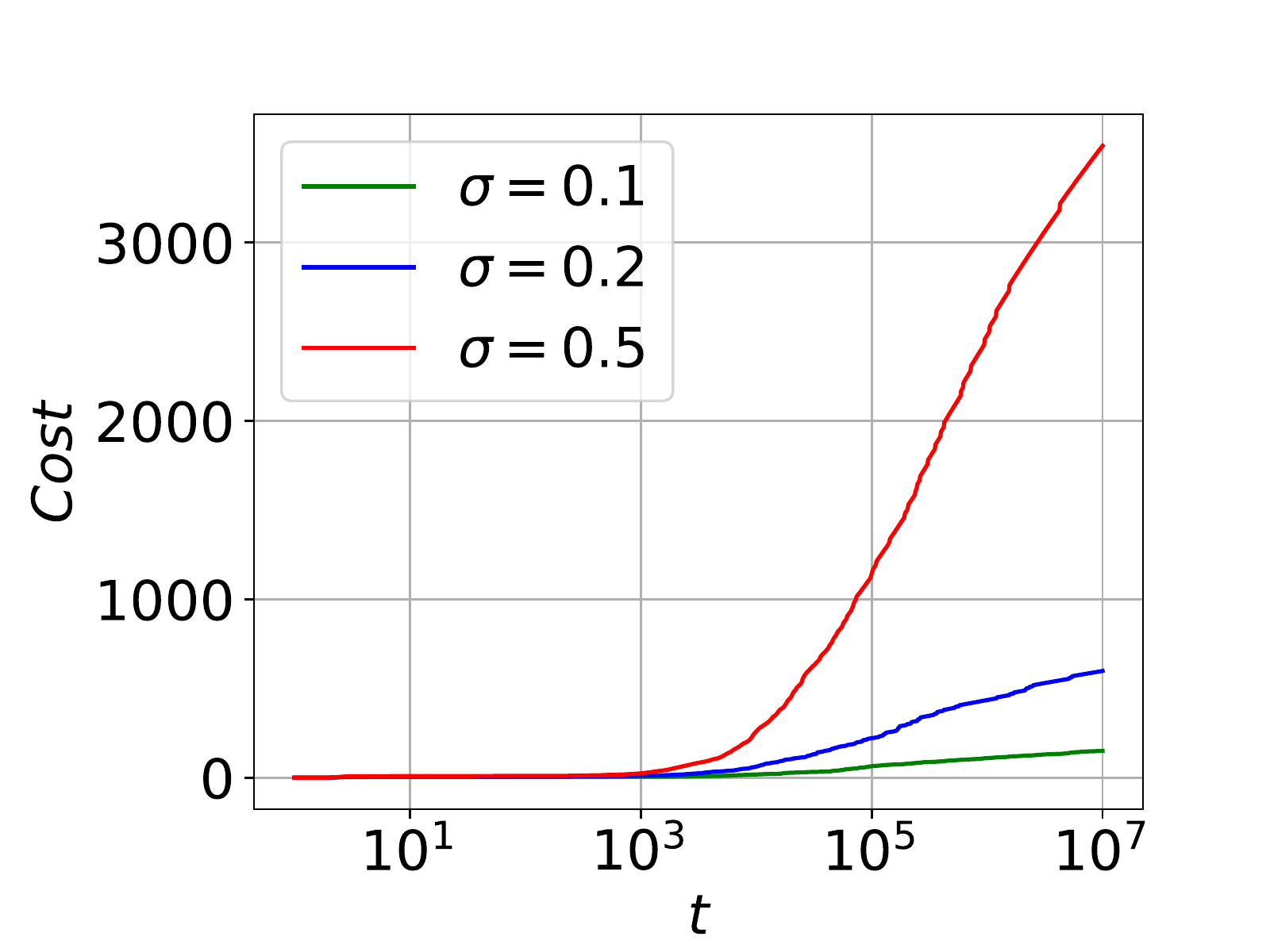}\label{exp:UCB_cost_VS}
}
\hspace{1em}
\subfloat[Target arm pulls $N_K(t)$]
{
\includegraphics[width=0.3\textwidth, height=0.22\textwidth]{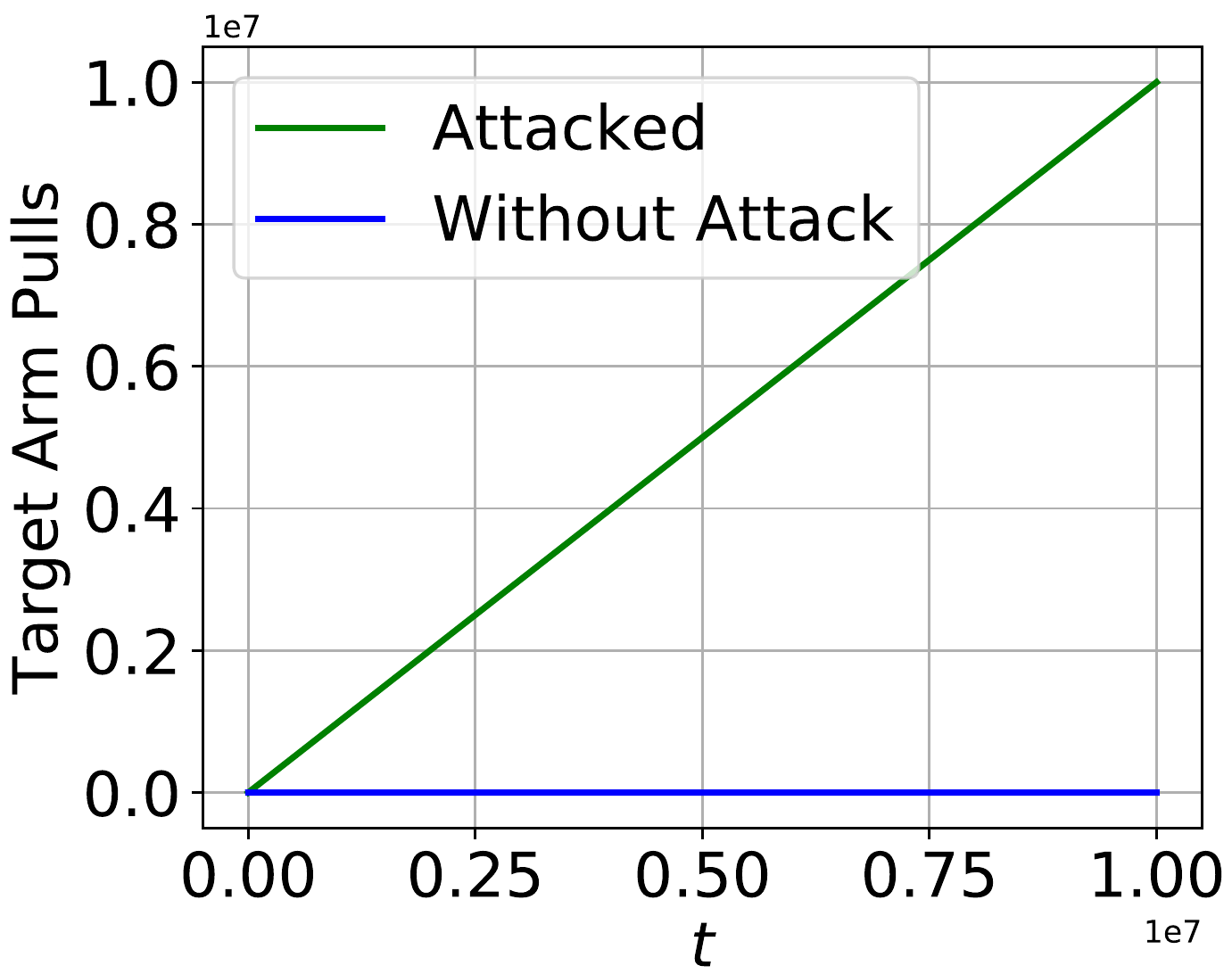}
\label{exp:UCB_TAP}
}
\caption{Attack on UCB learner.}\label{exp:UCB_cost}
\end{figure*}

%\clearpage
\section{Related Work}
\label{sec:related}

The literature on general adversarial learning is vast and covers ethics, safety, fairness, and legal concerns; see e.g.~\citet{joseph_nelson_rubinstein_tygar_2018} and \citet{goodfellow2014explaining}.
Related to MAB, there has been empirical evidence that suggests adversarial attacks can be quite effective, even in the more general multi-step reinforcement learning problems,
as opposed to the bandit case considered in this paper.  The learned policy may be lured to visit certain target states when adversarial examples are driven~\cite{lin17tatics}, or have inferior generalization ability when training examples are corrupted~\cite{huang17adversarial}.
There are differences, though.  In the first, non-stochastic setting~\cite{auer02nonstochastic,evendar09online}, the reward is generated by an adversary instead of a stationary, stochastic process.  However, the reward observed by the learner is still a \emph{real} reward, in that the learner is still interested in maximizing it, or more precisely, minimizing some notion of regret in reference to some reference policy~\citep{bubeck12regret}.  Another related problem is reward shaping (e.g., \citet{dorigo97robot}), where the reward received by the learner is modified, as in our paper.  However, those changes are typically done to \emph{help} the learner in various ways (such as promoting exploration), and are designed in a way not to change the optimal policy the learner eventually converges to~\cite{ng99policy}.

A concurrent work by~\citet{lykouris18stochastic} considers a complementary problem to ours. They propose a randomized bandit algorithm that is robust to adversarial attacks on the stochastic rewards. 
In contrast, our work shows that the existing stochastic algorithms are vulnerable to adversarial attacks. 
Note that their attack protocol is slightly different in that the attacker has to prepare attacks for all the arms before the learner chooses an arm. Furthermore, they have a different attack cost definition where the cost in a round is the \emph{largest} manipulation over the arms, regardless of which arm the learner selects afterwards.

Another concurrent work by~\citet{ma2018data} considers attacking stochastic contextual bandit algorithms. The authors show that for a contextual bandit algorithm which periodically updates the arm selection policy, an attacker can perform offline attack to force the contextual bandit algorithm to pull some pre-specified target arm for a given target context vector. Our work differs in that we consider online attack, which is performed on the fly rather than offline.

\section{Conclusions and Future Work}
\label{sec:con}

We presented a reward-manipulating attack on stochastic MABs.
We analyzed the attack against $\epsilon$-greedy and a generalization of the UCB algorithm, and proved that the attacker can force the algorithms to almost always pull a suboptimal target arm.
The cost of the attack is only logarithmic in time.
Given the wide use of MABs in practice, this is a significant security threat.

Our analysis is only the beginning.  We targeted $\epsilon$-greedy and UCB learners for their simplicity and popularity.
Future work may look into attacking Thompson sampling~\cite{thompson33onthelikelihood,agrawal12analysis}, linear bandits~\cite{abbasi11improved,agrawal13thompson}, and contextual bandits~\cite{li10contextual,agarwal14taming}, etc.
We assumed the reward attacks $\alpha_t$ are unbounded from above; new analysis is needed if an application's reward space is bounded or discrete.
It will also be useful to establish lower bounds on the cumulative attack cost.

Beyond the attack studied in this paper, there is a wide range of possible attacks on MABs. 
We may organize them along several dimensions:
\begin{compactitem}
\item{The attack goal: The attacker may force the learner into pulling or avoiding target arms, 
or worsen the learner's regret, or make the learner identify the wrong best-arm, etc.}
\item{The attack action: The attacker can manipulate the rewards or corrupt the context for contextual bandits, etc.}
\item{Online vs. offline: An online attacker must choose the attack action in real time; An offline attacker poisons a dataset of historical action-reward pairs in batch mode, then the learner learns from the poisoned dataset.}
%\item{Targeting stochastic vs. adversarial bandits: Attacks on adversarial bandits are expected and amount to \kwang{I have difficulty parsing/what it is implying. Perhaps explain a bit more?}realizing the worst-case analysis. Attacks on stochastic bandits are unexpected and require new analyses.}
\end{compactitem}

The combination of these attack dimensions presents fertile ground for future research into both bandit-algorithm attacks and the corresponding defense mechanisms. 

%\clearpage

\subsubsection*{Acknowledgments}

This work is supported in part by NSF 1837132, 1545481, 1704117, 1623605, 1561512, and the MADLab AF Center of Excellence FA9550-18-1-0166.

\bibliography{attack-nips}
\bibliographystyle{icml2017_kwang}

%--- acknowledgements

% \section*{References}
% 
% References follow the acknowledgments. Use unnumbered first-level
% heading for the references. Any choice of citation style is acceptable
% as long as you are consistent. It is permissible to reduce the font
% size to \verb+small+ (9 point) when listing the references. {\bf
%   Remember that you can use more than eight pages as long as the
%   additional pages contain \emph{only} cited references.}
% \medskip
% 
% \small

\clearpage
%\pdfoutput=1 % for arxiv submission

\onecolumn
%\begin{center}
%  {\Large\bf Supplementary Material}
%\end{center}          
\begin{center}
\section*{Supplementary Material}
\end{center}

\renewcommand{\thesubsection}{\Alph{subsection}}

\subsection{Details on the oracle and constant attack}

\paragraph{Logarithmic regret and the suboptimal arm pull counts.}

For simplicity, denote by $i^*$ the \emph{unique} best arm; that is, $i^* = \arg \max_{i=1,\ldots,K} \mu_i$.
We show that a logarithmic regret bound implies that the arm pull count of arm $i \neq i^*$ is at most logarithmic in $T$.
\begin{lem}\label{lem:regret-and-pull-counts}
Assume that a bandit algorithm enjoys a regret bound of $O(\log(T))$.
Then, $\EE N_i(T) = O(\log(T))$, $ \forall i \neq  i^*$.
\end{lem}
\begin{proof}
The logarithmic regret bound implies that for a large enough $T$ there exists $C>0$ such that $\sum_{i=1}^{K} \EE N_i(T) (\mu_{i^*} - \mu_i) \le C \log T$.
Therefore, for any $i \ne i^*$, we have $\EE N_i(T) (\mu_{i^*}-\mu_i) \le C \log T$, which implies that
\[
\EE N_i(T) \le \frac{C}{\mu_{i^*}-\mu_i} \log T = O(\log T)\,.
\]
\iffalse
We employ proof by contradiction.
Without loss of generality, suppose that arm $a$ satisfies $ \EE N_a(T) = \omega(\log T) $.
  Then, 
  \begin{align*}
  \EE N_a(T) (\mu_{i^*} - \mu_a) + \sum_{i \not\in \{i^*, a\}}^K \EE N_i(T)(\mu_{i^*} - \mu_i) &\le C \log T
  \\\implies (\mu_{i^*} - \mu_a) \omega(\log(T)) &\le C\log T \;,
  \end{align*}
  which leads to a contradiction for a large enough $T$.
\fi
\end{proof}

\paragraph{Proof of Proposition~\ref{thm:oracle}}

By Lemma~\ref{lem:regret-and-pull-counts}, a logarithmic regret bound implies that the bandit algorithm satisfies $\EE N_i(T) = O(\log(T))$.
That is, for a large enough $T$, $\EE N_i(T) \le C_i \log(T)$ for some $C_i >0$.
Based on the view that the oracle attack effectively shifts the means $\mu_1,\cdots,\mu_K$, the best arm is now the $K$-th arm.
Then, $\EE N_K(T) = T - \sum_{i\neq K} \EE N_i(T) \ge T - \sum_{i\neq K} C_i \log T = T - o(T)$, which proves the first statement.

For the second statement, we notice that $\EE N_i(T) = C_i \log T$ for any $i \ne K$ and that we do not attack the $K$-th arm.
Therefore, 
\[
\EE\left[\sum_{t=1}^T |\alpha_t|\right] = \sum_{i=1}^{K-1} \EE N_i(T) \cdot \Delta^\eps_i \le \sum_{i=1}^{K-1} C_i \Delta^\eps_i \log T = O\left( \sum_{i=1}^{K-1} \Delta^\eps_i \log T \right)\,.
\]

\paragraph{Proof of Proposition~\ref{thm:constant}}

By Lemma~\ref{lem:regret-and-pull-counts}, a logarithmic regret bound implies that the bandit algorithm satisfies $\EE N_i(T) = O(\log(T))$.
Note that the constant attack effectively shifts the means of all the arms by $A'$ except for the $K$-th arm.
Since $A' > \max_i \Delta_i$, the best arm is now the $K$-th arm.
Then, $\EE N_K(T) = T - \sum_{i=1}^{K-1} \EE N_i(T) \ge T - \sum_{i=1}^{K-1} C_i \log T = T - o(T)$, which proves the first statement.

For the second statement, we notice that $\EE N_i(T) = C_i \log T$ for any $i \neq K$, and we do not attack the $K$-th arm.
Therefore, 
\[
\EE\left[\sum_{t=1}^T |\alpha_t|\right] = \sum_{i=1}^{K-1} \EE N_i(T) \cdot A' \le A' \sum_{i=1}^{K-1} C_i \log T = O(A' \cdot \log T)\,.
\]

\paragraph{The best $\eps$ for Alice's oracle attack}

Consider the case where Bob employs a near-optimal bandit algorithm such as UCB~\cite{auer02finite}, which enjoys $\EE N_i(T) = \Theta(1 + \Delta_i^{-2} \log T)$. 
When the time horizon $T$ is known ahead of time, one can compute the best $\eps$ ahead of time.
%From the proof of Proposition~\ref{thm:oracle}, we know that $\EE N_i(T) = O(\log(T))$.
%Note that Bob must pull each arm at least once to have a sublinear regret bound.
%Therefore, $\EE N_i(T)$ is of order $\log(T) + C$.
Hereafter, we omit unimportant constants for simplicity.
Since Alice employs the oracle attack, Bob pulls each arm $C + \eps^{-2} \log(T)$ times for some $C>0$ in expectation.
Assuming that the target arm is $K$, the attack cost is
\begin{align*}
  \sum_{i=1}^{K-1} \Delta^\eps_i \cdot (C + \eps^{-2} \log(T))
  = C \sum_{i=1}^{K-1} \Delta_i + (K-1)C\cdot \eps + \sum_{i=1}^{K-1} \lt(  \fr{\Delta_i}{\eps^2} + \fr{1}{\eps} \rt) \log T
\end{align*}
To balance the two terms, one can see that $\eps$ has to grow with $T$ and the term $\Delta_i/\eps^2$ is soon dominated by $1/\eps$.
Thus, for large enough $T$ the optimal choice of $\eps$ is $\sqrt{C \log(T)}$, which leads to the attack cost of $O(K\sqrt{C\log T})$.

\subsection{Details on attacking the \texorpdfstring{$\eps$}{}-greedy strategy }

\begin{lem}
  \label{lem:beta}
  For $\delta\le1/2$, the $\beta(N)$ defined in~\eqref{eq:beta} is monotonically decreasing in $N$. %\frac{K}{3}(\frac{\pi}{e})^2$.
\end{lem}
\begin{proof}
  It suffices to show that $f(x)=\frac{2\sigma^2}{x}\log \frac{\pi^2Kx^2}{3\delta}$ is decreasing for $x\ge 1$. Note that $\delta\le 1/2 \le \frac{K}{3}(\frac{\pi}{e})^2$, thus for $x\ge1$ we have
  \begin{eqnarray*}
    f^\prime(x)&=& -\frac{2\sigma^2}{x^2}\log \frac{\pi^2Kx^2}{3\delta}+\frac{2\sigma^2}{x}\frac{3\delta}{\pi^2Kx^2}\frac{2\pi^2Kx}{3\delta} \nonumber \\
    &=&\frac{2\sigma^2}{x^2}(2-\log \frac{\pi^2Kx^2}{3\delta})\le \frac{2\sigma^2}{x^2}(2-\log \frac{\pi^2K}{3\delta}) \nonumber \\
    &\le& \frac{2\sigma^2}{x^2}(2-\log e^2)=0.
  \end{eqnarray*}
\end{proof}
\vspace{-1em}
\paragraph{Proof of Corollary~\ref{cor:sublinear}}
When $T$ is larger than the following threshold:
$$\fr{K+1}{K}(\sum_{t=1}^T \eps_t) + \sqrt{12 \log(K/\dt)(\fr{K+1}{K}\sum_{t=1}^T \eps_t)},$$
we have $\tilN_K(T) \ge \tilN (T)$.
Because $\beta(N)$ is decreasing in $N$,
\begin{equation}
\tilN(T) \beta(\tilN(T)) + 3 \tilN(T) \beta(\tilN_K(T)) \le 4 \tilN(T) \beta(\tilN(T)).
\label{eq:corupper1}
\end{equation}
Due to the the exploration scheme of the strategy,
$$\sum_{t=1}^T \epsilon_t = cK \sum_{t=1}^T 1/t \le cK(\log(T)+1).$$
Thus by the definition of $\tilde N(T)$,
\begin{eqnarray}
\tilde N(T) &\le& c(\log T+1) + \sqrt{3 \log\left(\frac{K}{\delta}\right)c(\log T + 1)}. \nonumber 
\end{eqnarray}
For sufficiently large $T$, there exists a constant $c_2$ depending on $c, K, \delta$ to further upper bound the RHS as follows:
\begin{equation}
c(\log T+1) + \sqrt{3 \log\left(\frac{K}{\delta}\right)c(\log T + 1)} \le c_2 \log T \defeq \breve N(T).
\label{eq:corupper2}
\end{equation}
Since $N \beta(N)$ is increasing in $N$, combining~\eqref{eq:corupper1} and~\eqref{eq:corupper2}
we have for sufficiently large $T$,
$$\tilN(T) \beta(\tilN(T)) + 3 \tilN(T) \beta(\tilN_K(T)) \le 4 \breve N(T) \beta(\breve N(T)).$$
Plugging this upper bound into Theorem~\ref{thm:egreedy-main},
\begin{eqnarray}
\lefteqn{\sum_{t=1}^T \alpha_t < \lt(\sum_{i=1}^K \Delta_i \rt) \breve N(T) + 4(K-1) \breve N(T) \beta(\breve N(T))} \nonumber\\
&=& c_2 \lt(\sum_{i=1}^K \Delta_i \rt) \log T + \sqrt{32 c_2} (K-1) \sigma \cdot \sqrt{ \log T \lt(2\log\log T + \log\frac{\pi^2 K c_2^2}{3 \delta} \rt) }.
\end{eqnarray}

\paragraph{Proof of Lemma~\ref{lem:band}}
Let $\{X_j\}_{j=1}^\infty$ be a sequence of \emph{i.i.d.} $\sigma^2$-sub-Gaussian random variables with mean $\mu$.
Let $\hat\mu^0_N = \frac{1}{N}\sum_{j=1}^N X_j$. 
By Hoeffding's inequality
\begin{equation*}
\PP(|\hat\mu^0_N - \mu|\ge \eta) \le 2 \exp\left(-\frac{N\eta^2}{2\sigma^2}\right).
\end{equation*}
Define $\delta_{iN} \defeq \frac{6\delta}{\pi^2 N^2 K}$.
Apply union bound over arms $i$ and pull counts $N \in \mathbb N$, 
\begin{align*}
\PP\left(\exists i, N: |\hat\mu^0_{i,N} - \mu_i|\ge \beta(N) 
\right) \le \sum_{i=1}^K \sum_{N=1}^\infty \delta_{iN}  = \delta. 
\end{align*}

\paragraph{Proof of Lemma~\ref{lem:alwaysK}}

We show by induction that at the end of any round $t \ge K$ Algorithm~\ref{alg:attackeps} maintains the invariance
\begin{align}
\hat\mu_{K}(t) > \hat\mu_{i}(t),\quad \forall i<K,
\label{eq:invariance}
\end{align}
which forces the learner to pull arm $K$ if $t+1$ is an exploitation round.

Base case: By definition the learner pulls arm $K$ first, then all the other arms once.
During round $t=2 \ldots K$ the attack algorithm ensures 
$\hat\mu_{i}(t) \le \hat\mu_K(t) - 2 \beta(1)<\hat\mu_K(t)$ for arms $i<K$, trivially satisfying~\eqref{eq:invariance}.

Induction: Suppose~\eqref{eq:invariance} is true for rounds up to $t-1$.
Consider two cases for round $t$: 

If round $t$ is an exploration round and $I_t \neq K$ is pulled, then only $\hat\mu_{I_t}(t)$ changes; the other arms copy their empirical mean from round $t-1$.
The attack algorithm ensures $\hat\mu_{K}(t) \ge \hat\mu_{I_t}(t) + 2 \beta(N_{K}(t)) > \hat\mu_{I_t}(t)$.
Thus~\eqref{eq:invariance} is satisfied at $t$.

Otherwise either $t$ is exploration and $K$ is pulled; or $t$ is exploitation -- in which case $K$ is pulled because by inductive assumption~\eqref{eq:invariance} is satisfied at the end of $t-1$.
Regardless, this arm $K$ pull is not attacked by Algorithm~\ref{alg:attackeps} and its empirical mean is updated by the pre-attack reward. 
We show this update does not affect the dominance of $\hat\mu_{K}(t)$.
Consider any non-target arm $i<K$.
Denote the last time $\hat\mu_i$ was changed by $t'$.
Note $t'<t$ and $N_{K}(t')<N_{K}(t)$.
At round $t'$, Algorithm~\ref{alg:attackeps} ensured that 
$\hat\mu_{i}(t') \le \hat\mu_{K}(t') - 2 \beta(N_{K}(t'))$.
We have:
\begin{align*}
  \hat\mu_{K}(t) \nonumber
  &= \hat\mu^0_{K}(t) \quad & \mbox{(arm $K$ never attacked)} \nonumber \\
  &> \mu^0_{K} - \beta(N_{K}(t)) \quad & \mbox{(\eqref{eq:band} lower bound)} \nonumber \\
  &> \mu^0_{K} - \beta(N_{K}(t')) \quad & \mbox{(\lemref{lem:beta})} \nonumber \\
  &> \hat\mu_{K}(t') - 2\beta(N_{K}(t')) \quad & \mbox{(\eqref{eq:band} upper bound)} \nonumber \\
  &\ge \hat\mu_{i}(t') \quad & \mbox{(Algorithm~\ref{alg:attackeps})} \nonumber \\
  &= \hat\mu_{i}(t)\,. & 
\end{align*}
Thus~\eqref{eq:invariance} is also satisfied at round $t$.

\paragraph{Proof of Lemma~\ref{lem:alphaupperbound}}

Without loss of generality assume in round $t$ arm $i$ is pulled and the attacker needed to attack the reward (i.e. $I_t=i$ and $\alpha_t>0$).
By definition~\eqref{eq:alpha},
{\small
  \begin{eqnarray}
  \alpha_t &=& \hat\mu_{i}(t-1)N_{i}(t-1) + r_t^0 - \left(\hat\mu_{K}(t) - 2 \beta(N_{K}(t))\right) N_{i}(t)  \nonumber\\
  &=& \sum_{s\in \tau_i(t-1)} (r_s^0 - \alpha_s)  + r_t^0 - \left(\hat\mu_{K}(t) - 2 \beta(N_{K}(t))\right) N_{i}(t)  \nonumber\\
  &=& \sum_{s\in \tau_i(t)} r_s^0 -  \sum_{s\in \tau_i(t-1)} \alpha_s   - \left(\hat\mu_{K}(t) - 2 \beta(N_{K}(t))\right) N_{i}(t). \nonumber 
  \end{eqnarray}
}
Therefore, the cumulative attack on arm $i$ is
\begin{eqnarray}
\sum_{s\in \tau_i(t)} \alpha_s &=& \sum_{s\in \tau_i(t)} r_s^0 - \left(\hat\mu_{K}(t) - 2 \beta(N_{K}(t))\right) N_{i}(t) \nonumber\\  
&=& \left(\hat\mu_i^0(t) - \hat\mu_{K}(t) + 2 \beta(N_{K}(t))\right) N_{i}(t).\nonumber 
\end{eqnarray}
One can think of the term in front of $N_i(t)$ as the amortized attack cost against arm $i$.
By \lemref{lem:band}, 
\begin{eqnarray*}
  \hat\mu_i^0(t) &<& \mu_i + \beta(N_i(t)) \\
  \hat\mu_{K}(t) = \hat\mu^0_{K}(t) &>& \mu_{K} - \beta(N_{K}(t)) 
\end{eqnarray*}
Therefore,
\begin{eqnarray*}
  \sum_{s:I_s=i}^{t} \alpha_s &<& \left(\mu_i - \mu_{K} + \beta(N_i(t)) + 3 \beta(N_{K}(t))\right) N_{i}(t)\nonumber \\
  &\le& \left(\Delta_i + \beta(N_i(t)) + 3 \beta(N_{K}(t))\right) N_{i}(t).
  \label{eq:tmp1}
\end{eqnarray*}
The last inequality follows from the gap definition $\Delta_i := [\mu_i - \mu_{K}]_+$.
% Since $\log(z)/z \le 1/e$ for $z> 0$, one can show that $\forall N, \beta(N)< 2\sigma \sqrt{\frac{\pi}{\sqrt{3\delta}e}}$.
% Plugging in~\eqref{eq:tmp1} proves the lemma.

\paragraph{Proof of Lemma~\ref{lem:Nitupperbound}}

Fix a non-target arm $i<K$. 
Let $X_t$ be the Bernoulli random variable for round $T$ being arm $i$ pulled.  
Then,
\begin{eqnarray*}
  N_i(T) &=& \sum_{t=1}^T X_t \\
  \EE[X_t] &=& \frac{\eps_t}{K} \\
  \VV[X_t] &=& \frac{\eps_t}{K}(1-\frac{\eps_t}{K}) < \frac{\eps_t}{K}\,.
\end{eqnarray*}

Since $X_t$'s are independent random variables, we may apply Lemma~9 of \citet{agarwal14taming}, so that for any $\lambda\in[0,1]$, with probability at least $1-\delta/K$,
\begin{eqnarray*}
  \sum_{t=1}^T (X_t - \frac{\eps_t}{K}) 
  &\le& (e-2)\lambda\sum_{t=1}^T\VV[X_t] +\frac{1}{\lambda}\log\frac{K}{\delta}\\
  &<& (e-2)\lambda\sum_{t=1}^T \EE[X_t] + \frac{1}{\lambda}\log\frac{K}{\delta}\,.
\end{eqnarray*}
Choose $\lambda=\sqrt{\frac{\log(K/\delta)}{(e-2)\sum_{t=1}^T\EE[X_t]}}$, and we get that
\begin{eqnarray*}
  \sum_{t=1}^T X_t &<& \sum_{t=1}^T \frac{\eps_t}{K} + 2\sqrt{(e-2)\sum_{t=1}^T\frac{\eps_t}{K}\log\frac{K}{\delta}} \\
  &<& \sum_{t=1}^T \frac{\eps_t}{K} + \sqrt{3\sum_{t=1}^T\frac{\eps_t}{K}\log\frac{K}{\delta}} \defeq \tilde N(T)\,.
\end{eqnarray*}
The same reasoning can be applied to all non-target arm $i<K$.  
\footnote{Note the upper bound above is valid for $T$ such that $\sum_{t=1}^T \eps_t \ge \fr{K}{e-2}\log(K/\dt)$ only as otherwise $\lambda$ is greater than 1.
  One can get rid of such a condition by a slightly looser bound.
  Specifically, using $\lambda=1$ gives us a bound that holds true for all $T$.
  We then take the max of the two bounds, which can be simplified as
  $
  \sum_{t=1}^T X_t
  < (e-1)\sum_{t=1}^T \frac{\epsilon_t}{K} + \sqrt{3\sum_{t=1}^T\frac{\epsilon_t}{K}\log\frac{K}{\delta}} + \log\fr{K}{\dt} \;.
  $
  The condition on $T$ in Theorem~\ref{thm:egreedy-main} can be removed using this bound.
  However, by keeping the mild assumption on $T$ we keep the exposition simple.
}

The case with the target arm is similar, with the only change that $\EE[X_t] > 1-\epsilon_t$ and $\VV[X_t] < \epsilon_t$, leading to the lower bound:
\[
N_K(T) > T - \sum_{t=1}^T\epsilon_t - \sqrt{3\sum_{t=1}^T\epsilon_t\log\frac{K}{\delta}} =: \tilde N_K(T)\,.
\]
Finally, a union bound is applied to all $K$ arms to complete the proof.

\subsection{Details on attacking the UCB strategy}

\paragraph{Proof of \lemref{lem:ni-upper-bound}}

Fix some $t\ge2K$.  If $N_i(t) \le 2$ for all $i<K$, then $N_K(t) \ge 2$, which implies $N_i(t) \le \min\{N_K(t), 2\}$. Thus, \eqref{eqn:ni-ub} holds trivially and we are done.

Now fix any $i<K$ such that $N_i(t)>2$.  
As the desired upper bound is nondecreasing in $t$, we only need to prove the result for $t$ where $I_t=i$.  
Let $t'$ be the previous time where arm $i$ was pulled.  
Note that $t'$ satisfies $K < t' < t$ as $N_i(t)>2$, so the attacker has started attacking at round $t'$.
This implies that $N_i(t'-1) + 1 = N_i(t') = N_i(t-1) = N_i(t) - 1$.

On one hand, it is clear that after attack $\alpha_{t'}$ was added at round $t'$, the following holds:
\begin{equation}
\hmu_i(t') \le \hmu_K(t') - 2 \beta(N_K(t'))-\Delta_0\,. \label{eqn:after_tp}
\end{equation}

On the other hand, at round $t$, it must be the case that
\[
\hmu_i(t-1) + 3\sigma\sqrt{\frac{\log t}{N_i(t-1)}} \ge \hmu_K(t-1) + 3\sigma\sqrt{\frac{\log t}{N_K(t-1)}}\,,
\]
which is equivalent to
\[
\hmu_i(t') + 3\sigma\sqrt{\frac{\log t}{N_i(t')}} \ge \hmu_K(t-1) + 3\sigma\sqrt{\frac{\log t}{N_K(t-1)}}\,.
\]
Therefore,
\begin{eqnarray*}
\lefteqn{3\sigma\sqrt{\frac{\log t}{N_i(t')}} - 3\sigma\sqrt{\fr{\log t}{N_K(t-1)}} \ge \hmu_K(t-1) - \hmu_i(t')} \\
&\ge& \hmu_K(t-1) - \hmu_K(t') + 2 \beta(N_K(t')) + \Delta_0 \\
&\ge& \Delta_0\,,
\end{eqnarray*}
where we have used \eqnref{eqn:after_tp} in the second inequality, the condition in event $E$ as well as \lemref{lem:beta} in the third.  
Since $\Delta_0 > 0$, we can see that $N_i(t') < N_K(t-1)$, and thus
\begin{equation}
N_i(t) = N_i(t') + 1 \le N_K(t-1) = N_K(t)\,. \label{eqn:ni-ub1}
\end{equation}

Furthermore, since $3\sigma\sqrt{\fr{\log t}{N_K(t-1)}} > 0$, we have $3\sigma\sqrt{\frac{\log t}{N_i(t')}} > \Delta_0$, which implies
\begin{equation}
N_i(t) = 1 + N_i(t') \le 1 + \frac{9\sigma^2}{\Delta_0^2}\log t\,. \label{eqn:ni-ub2}
\end{equation}
Combining \eqref{eqn:ni-ub1} and \eqref{eqn:ni-ub2} gives the desired bound \eqref{eqn:ni-ub}.

\paragraph{Proof of \lemref{lem:total-cost}}

Fix any $i<K$.  As the desired upper bound is increasing in $t$, we only need to prove the result for $t$ where $I_t=i$ and $\alpha_t>0$.
%Furthermore, if $\alpha_t=0$, then
%\[
%\sum_{s \in \tau_i(t)} \alpha_s = \sum_{s \in \tau_i(t-1)} \alpha_s\,,
%\]
%%so the bound holds trivially as the terms corresponding to arm $i$ on the right-hand side of the bound is monotonically increasing with $t$.  We therefore 
%so we only have to consider the case where $\alpha_t > 0$.
%
It follows from \eqref{eqn:ucb-attack} that, 
\begin{eqnarray*}
\frac{1}{N_i(t)} \sum_{s\in \tau_i(t)} \alpha_s &=& \hmu^0_i(t) - \hmu_K(t-1) + 2 \beta(N_K(t-1)) + \Delta_0\,.
\end{eqnarray*}
Since event $E$ holds, we have
\begin{eqnarray*}
\frac{1}{N_i(t)} \sum_{s\in \tau_i(t)} \alpha_s \le \Delta_i + \Delta_0 + \beta(N_i(t)) + 3 \beta(N_K(t-1))\,.
\end{eqnarray*}
The proof is completed by observing $N_K(t-1)=N_K(t)$, $N_i(t) \le N_K(t)$ (\lemref{lem:ni-upper-bound}) and \lemref{lem:beta}.

\subsection{Simulations on Heuristic Constant Attack}
\label{sup:constant}
We run simulations on $\epsilon$-greedy and UCB to illustrate the heuristic constant attack algorithm. 
The bandit has two arms, where the reward distributions are $\mathcal{N}(1, 0.1^2)$ and $\mathcal{N}(0, 0.1^2)$ respectively, thus $\max_i \Delta_i = \mu_1 -\mu_2 =1$. 
Alice's target arm is arm 2. In our experiment, Alice tried two different constants for attack: $A = 1.2$ and $A= 0.8$, one being greater and the other being smaller than $\max_i \Delta_i$. 
We run the attack for $T = 10^4$ rounds. Fig.~\ref{cattack:eps} and Fig.~\ref{cattack:UCB} show Alice's cumulative attack cost and Bob's number of target arm pulls $N_K(t)$ for $\epsilon$-greedy and UCB. 
Note that if $A>\max_i \Delta_i$, then $N_K(t)\approx t$, which verifies that Alice succeeds with the heuristic constant attack.
At the same time, pushing up the target arm would incur linear cost; while dragging down the non-target arm achieves logarithmic cost.
In summary, Alice should use an $A$ value larger than $\Delta$, and should drag down the expected reward of the non-target arm by amount $A$.
\begin{figure}[H]
  \centering
  \subfloat[push up the target arm: $\alpha_t = \one\{I_t = 2\}\cdot(-A)$]{
    \includegraphics[width=0.25\textwidth,height=0.2\textwidth]{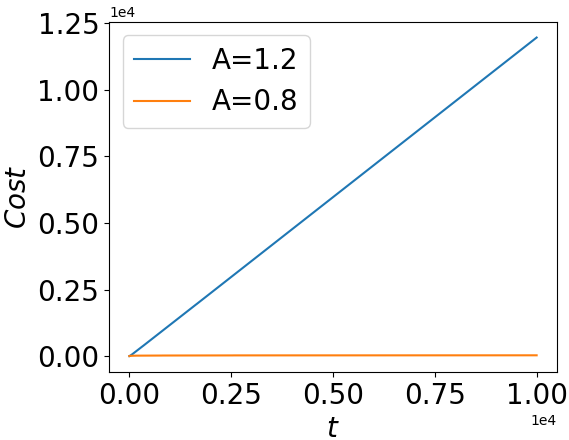}
    \includegraphics[width=0.25\textwidth,height=0.2\textwidth]{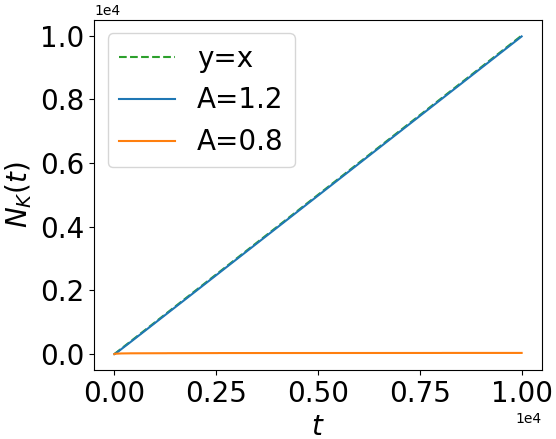}
  }
  \subfloat[drag down the non-target arm: $\alpha_t = \one\{ I_t \neq 2 \} \cdot A$]{
    \includegraphics[width=0.25\textwidth,height=0.2\textwidth]{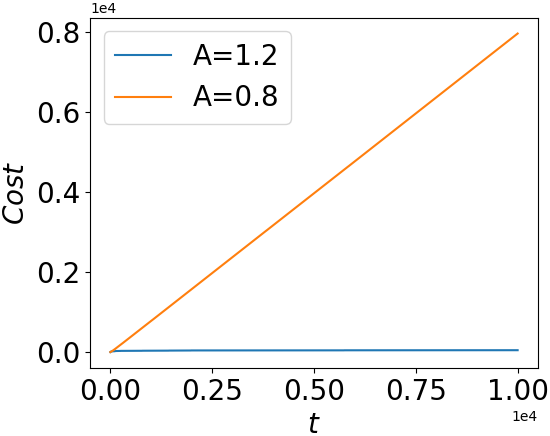}
    \includegraphics[width=0.25\textwidth,height=0.2\textwidth]{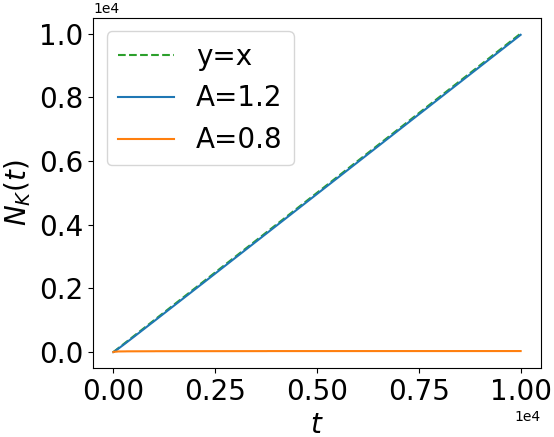}
  }
  \caption{Constant attack on $\epsilon$-greedy}
  \label{cattack:eps}
\end{figure}

\begin{figure}[H]
  \centering
  \subfloat[push up the target arm]{
    \includegraphics[width=0.25\textwidth,height=0.2\textwidth]{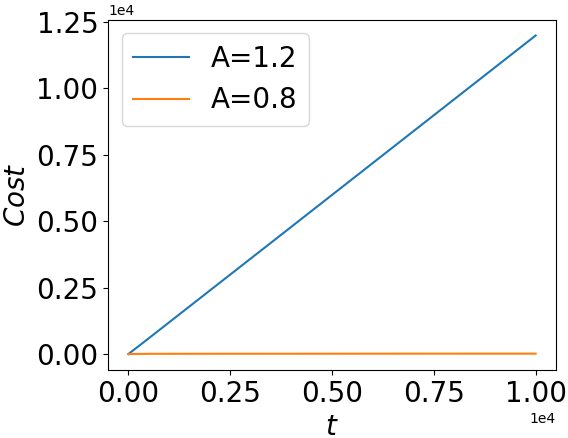}
    \includegraphics[width=0.25\textwidth,height=0.2\textwidth]{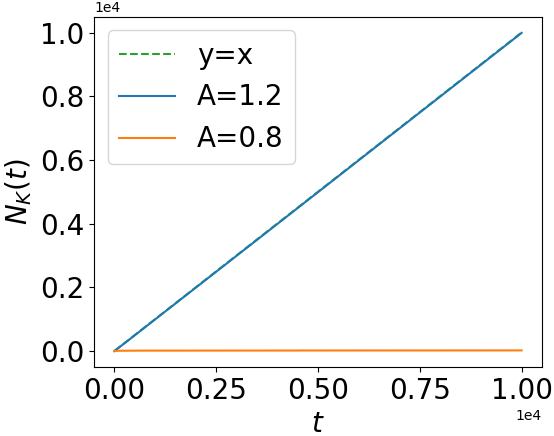}
  }
  \subfloat[drag down the non-target arm]{
    \includegraphics[width=0.25\textwidth,height=0.2\textwidth]{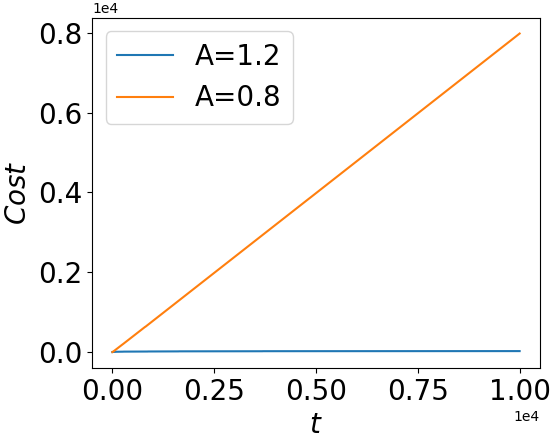}
    \includegraphics[width=0.25\textwidth,height=0.2\textwidth]{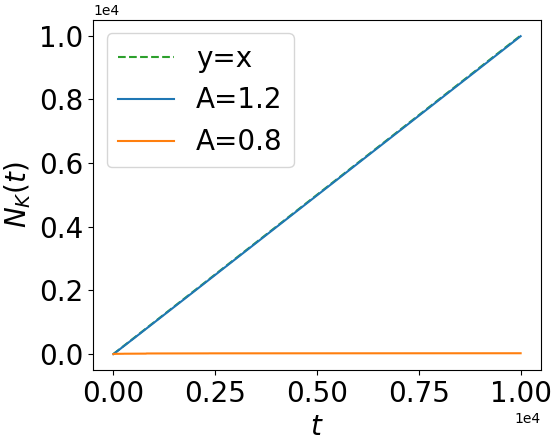}
  }
  \caption{Constant attack on UCB1}
  \label{cattack:UCB}
\end{figure}

\end{document}